\documentclass[letterpaper]{article} 
\usepackage{aaai2027}  
\usepackage[hyphens]{url}  
\usepackage{graphicx} 
\usepackage{natbib}  
\usepackage{caption} 
\usepackage{amsmath}
\usepackage{amssymb}
\usepackage{amsthm}
\usepackage{booktabs}
\usepackage{makecell}
\usepackage{pifont}
\usepackage{tikz}
\usetikzlibrary{arrows.meta,positioning,backgrounds,calc}
\definecolor{oiVermillion}{HTML}{D55E00}
\definecolor{oiOrange}{HTML}{E69F00}
\definecolor{oiGreen}{HTML}{009E73}
\definecolor{oiBlue}{HTML}{0072B2}

\theoremstyle{plain}
\newtheorem{proposition}{Proposition}

\newcommand{\silr}{\textsc{SiLR}}
\newcommand{\psione}{\psi_1}
\newcommand{\psitwo}{\psi_2}
\newcommand{\psithree}{\psi_3}

\title{\silr{}: Structure-Preserving Admission and Process Reward for LLM Tool Agents}
\author{
    Chenyu Zhou\textsuperscript{\rm 1}\equalcontrib,
    Qiliang Jiang\textsuperscript{\rm 2}\equalcontrib,
    Shuning Wu\textsuperscript{\rm 3},
    Xu Zhou\textsuperscript{\rm 3}
}
\affiliations{
    \textsuperscript{\rm 1}School of Engineering, Institute of Science Tokyo, Japan\\
    \textsuperscript{\rm 2}College of Control Science and Engineering, Zhejiang University, China\\
    \textsuperscript{\rm 3}Department of Electrical and Computer Engineering, National University of Singapore, Singapore\\
    zhou.c.76d6@m.isct.ac.jp, zhouxu\_nus@u.nus.edu
}

\begin{document}
\nocopyright
\maketitle

\begin{abstract}
A runtime gate for an LLM tool agent is usually cast as a filter. In a ReAct loop, however, a rejected proposal is followed by another at the same state: the gate is a \emph{search operator} over the proposal stream, and its admission criterion shapes which trajectories are reachable. We study \emph{post-violation recovery admission}, where progress must be admitted while the system is still in violation, and identify the \emph{scalar projection trap}: an aggregate-score gate can accept a locally improving proposal and commit the trajectory to a residual plateau. \silr{} instead shadow-executes each proposal and admits it under a product order over the branch-level violation state (overloaded-branch support and per-branch severity). We prove that no scalar surrogate is sound for this order, so the failure is a representational impossibility rather than a matter of threshold tuning. On mined Gym-ANM scenarios, \silr{} recovers $\mathbf{21/21}$ multi-action episodes, against $0/21$ for terminal admission and $9/21$ for the best scalar gate, an advantage significant across the full $24$-scenario benchmark. The terminal-versus-structured dichotomy holds across three model families and in CityLearn. Because admission rests on deterministic simulation, the LLM lies outside the trust boundary and admission authority resides in the verifier: a magnitude-redistribution attack that defeats both scalar and support-only baselines is contained only by the full per-branch predicate. With two constraint families active, every tested scalar projection admits physically unsafe actions, and the support-only gate admits the largest fraction ($63.2\%$ of $42{,}410$ unsafe actions; product order $0$). In the hardest dual-family traces, scalar gates recover only through that unsafe class. Reused as a GRPO process reward, the same signal outperforms its count projection in every mined scenario and is the only tested reward under which the ungated policy exceeds the untrained base ($0.844$ vs.\ $0.778$). Scalar projection loses the violation geometry at both design points; only the full product order is structurally sufficient.
\end{abstract}

\section{Introduction}
\label{sec:intro}

Critical-infrastructure operations---power dispatch, datacenter scheduling, financial settlement---increasingly rely on tool-using large language model (LLM) agents~\citep{gridagent2025}. Recent runtime-enforcement systems (SARC~\citep{sarc2026}, AgentSpec~\citep{agentspec2026}, GuardAgent~\citep{guardagent2025}, ShieldAgent~\citep{shieldagent2025}, VeriGuard~\citep{veriguard2025}) advance the runtime governance of these agents. They share an implicit premise: the system starts \emph{inside} the safe set, and the gate's job is to keep it there. Infrastructure under autonomous control, however, will sometimes start \emph{outside} it, already in violation and needing several actions to recover. The enforcement question then inverts: the gate must decide whether an action that still leaves the system unsafe nonetheless makes admissible progress toward recovery. This \emph{post-violation recovery admission} problem differs structurally from runtime prevention, and prior runtime governance for LLM tool agents does not address it. Runtime mediation intercepts unsafe actions, yet it rarely enables recovery after blocking an unsafe proposal~\citep{verifiertax2026}: enforcement is not recovery.

We frame the problem around \emph{admission semantics}: the formal criterion by which a runtime gate decides to apply or reject an LLM-proposed action. Our thesis is that this choice---not the gate's strictness---has first-order effects on whether multi-step recovery is possible. A terminal gate (admit only fully-recovered post-states) is safe but admits no intermediate step, deadlocking every multi-action episode on our mined Gym-ANM~\citep{henry2021gymanm} benchmark. The natural relaxation is a \emph{scalar admission gate}, admitting any proposal whose aggregate violation penalty is non-increasing within a small slack. It is representationally insufficient: relaxing the slack does not close the recovery gap.

The reason is the \textbf{scalar projection trap}, an admission-level instance of proxy gaming~\citep{skalse2022reward,krakovna2020specification}. In a ReAct loop~\citep{yao2023react} a rejected proposal is followed by another at the same state, so \emph{the gate is not a filter but a search operator over the agent's proposal stream}. Admitting on an aggregate score collapses the multi-component violation geometry and can commit the trajectory to a residual plateau. A gate that preserves the per-branch geometry instead rejects a scalar-improving proposal that violates the product order and keeps the search alive. The gap is representational, which we make precise as an impossibility result and confirm empirically: $\mathbf{21/21}$ structured vs.\ $9/21$ best scalar and $0/21$ terminal, significant across the full $24$-scenario benchmark.

\paragraph{Contributions.} Our unifying claim concerns one object: a product order over the branch-level violation state. It is the representation required at \emph{two} design points of an LLM tool agent, and projecting it to a scalar fails at both for the same representational reason.
\begin{itemize}\setlength{\itemsep}{1pt}
\item \textbf{The gate as a search operator, and the scalar projection trap.} We reframe the runtime admission gate not as a filter but as a search operator over the ReAct proposal stream. Its scalar instantiation collapses the violation geometry and commits the trajectory to a residual plateau.
\item \textbf{Structured recovery admission (design point~1).} We propose \silr{} (Structure-preserving In-the-Loop Recovery), a simulator-backed admission layer that shadow-executes each LLM tool proposal and admits it only when the product order is preserved, emitting graded \texttt{PASS}/\texttt{SAFE\_PROGRESS}/\texttt{FAIL} verdicts; we prove no scalar surrogate is sound with respect to this order.
\item \textbf{Empirical evidence across model families and domains.} Terminal admission deadlocks, and a scalar threshold sweep cannot close the recovery gap; plateau signatures and a model-predictive control (MPC) feasibility check confirm the mechanism is gate-induced; the dichotomy holds across three model families and a second domain (CityLearn).
\item \textbf{Trust-boundary containment under LLM compromise.} Because admission is grounded in deterministic simulation, admission authority moves from the LLM to the verifier; a magnitude-redistribution attack that defeats both scalar admission and the Grid-Agent (support-only) baseline is contained only by the full per-branch predicate. Under dual-family stress this becomes a safety--liveness split~\citep{alpern1985liveness}. Despite near-ceiling single-constraint-family recovery, support-only admits the largest fraction of unsafe actions ($63.2\%$ of $42{,}410$), whereas the product order admits none.
\item \textbf{The same violation geometry as a process reward (design point~2).} Reused as a GRPO signal, the product order yields a near-oracle process reward: deterministic, scalable process supervision with no human step labels. It is also the only tested reward under which the policy stays above the untrained baseline once the gate is removed; no tested scalar projection of the reward attains full recovery in both multi-family stress regimes we test.
\end{itemize}

\section{Related Work}
\label{sec:related}

\paragraph{Shielding and runtime guards.}
Shielding~\citep{alshiekh2018safe} synthesizes a finite-state shield from a temporal-logic specification that intercepts policy outputs at runtime. Model-Predictive Shielding~\citep{bastani2019mps}, Recovery RL~\citep{thananjeyan2021recovery}, and RL power-grid shielding~\citep{hrlshield2026} relax binary terminal gates by admitting recoverability-set states or by screening actions with forward simulation. All of them, however, assume a fixed action space and require a backup or recovery policy, or a dynamics-model lookahead. A parallel line adds runtime enforcement to tool-using LLM agents through architectural gates, guard agents, or rule languages---SARC's Pre-Action Gate~\citep{sarc2026} (the closest binary baseline), AgentSpec~\citep{agentspec2026}, GuardAgent~\citep{guardagent2025}, ShieldAgent~\citep{shieldagent2025}, and VeriGuard~\citep{veriguard2025}, yet none target post-violation recovery admission: a terminal predicate preserves safety but deadlocks multi-step recovery. \silr{} instead shields black-box LLM tool-call dispatch by single-step shadow simulation under a product order on the branch-level violation state, at constant per-step cost.

\paragraph{Grid-Agent and scalarization.}
Grid-Agent~\citep{gridagent2025} is the closest same-domain work: it couples LLM agents with sandboxed power-flow validation, rollback of negative actions, and monotonic-progress behavior for power-grid control. The mechanisms differ in kind. Grid-Agent's check is a \emph{post-hoc} rollback that reverts an action unless the violation set is resolved without introducing new violations. This check is a support-level criterion with no per-branch severity guard. \silr{}'s product order is instead a \emph{pre-action} admission predicate that never applies a proposal that violates the product order and emits graded verdicts reusable as a training signal. Scalarization is classically studied as a \emph{static} representational problem~\citep{miettinen1999nonlinear}. The scalar projection trap instead concerns the \emph{behavior} of the runtime gate, which acts as a commitment operator over the ReAct proposal stream, the admission-level counterpart of reward hacking~\citep{skalse2022reward,krakovna2020specification}. On the training side, the same verifier provides deterministic, simulator-based process supervision~\citep{lightman2024prm,uesato2022process} in place of human step labels. Prompt-injection defenses~\citep{greshake2023injection} sanitize inputs; \silr{} instead assumes a compromised LLM and enforces safety after the proposal. Table~\ref{tab:related} positions \silr{} against the closest systems along four axes: \emph{graded} admission, \emph{retained} multi-component violation geometry, \emph{black-box} LLM tool proposals, and no backup policy or dynamics lookahead. \silr{} is the only system that satisfies all four.

\begin{table}[t]
\centering
\caption{Positioning of \silr{} against the closest runtime-safety systems. ``Geom.'': retains the (support, severity) violation state vs.\ a scalar aggregate; ``No look.'': single-step shadow only, no recovery policy or dynamics rollout.}
\label{tab:related}
\footnotesize
\setlength{\tabcolsep}{2pt}
\begin{tabular*}{\columnwidth}{@{\extracolsep{\fill}}lcccc@{}}
\toprule
System & Graded & Geom. & \makecell{Black-\\box LLM} & \makecell{No look-\\ahead} \\
\midrule
SARC Pre-Action Gate        & \ding{55} & \ding{55} & \ding{51} & \ding{51} \\
AgentSpec          & \ding{55} & \ding{55} & \ding{51} & \ding{51} \\
GuardAgent / ShieldAgent & \ding{55} & \ding{55} & \ding{51} & \ding{51} \\
Grid-Agent             & \ding{55} & \ding{55}$^\dagger$ & \ding{51} & \ding{55}$^\ddagger$ \\
MPS / Recovery RL & \ding{55} & \ding{55} & \ding{55} & \ding{55} \\
\silr{} (ours)                              & \ding{51} & \ding{51} & \ding{51} & \ding{51} \\
\bottomrule
\end{tabular*}\\[2pt]
{\footnotesize $^\dagger$support-level validation, no per-branch severity guard; $^\ddagger$applies then reverts (post-hoc rollback).}
\end{table}

\section{Threat Model and Method}
\label{sec:method}

\begin{figure*}[t]
\centering
\resizebox{0.86\textwidth}{!}{%
\begin{tikzpicture}[
  box/.style={rounded corners=3pt, draw, align=center, inner sep=5pt, line width=0.8pt},
  trusted/.style={fill=oiBlue!7, draw=oiBlue!60},
  untrusted/.style={fill=oiVermillion!9, draw=oiVermillion!65},
  hero/.style={fill=oiGreen!12, draw=oiGreen!75, line width=1.3pt},
  arr/.style={-{Stealth[length=2.2mm]}, line width=0.9pt, gray!35!black},
  admit/.style={-{Stealth[length=2.6mm]}, line width=1.4pt, oiGreen!58!black},
  fail/.style={-{Stealth[length=2.2mm]}, line width=0.9pt, oiVermillion!85!black, dashed},
  lbl/.style={font=\scriptsize, inner sep=1.5pt},
  lblg/.style={font=\scriptsize\bfseries, text=oiGreen!42!black, inner sep=1.5pt},
  lblr/.style={font=\scriptsize, text=oiVermillion!80!black, inner sep=1.5pt}
]
\node[box,untrusted,text width=2.5cm] (llm) at (0,0)
  {\textbf{LLM agent}\\[1pt]\scriptsize possibly compromised};
\node[box,trusted,text width=3.0cm] (shadow) at (5.0,1.5)
  {\textbf{Shadow verifier}\\[1pt]\scriptsize simulate on a deep copy};
\node[box,hero,text width=3.7cm] (gate) at (10.7,1.5)
  {\textbf{Product-order gate}\\[1pt]\scriptsize admit iff $\Phi{=}(S,\sigma)$ non-increasing};
\node[box,trusted,text width=2.9cm] (live) at (7.8,-1.55)
  {\textbf{Live simulator $M$}};
\begin{scope}[on background layer]
  \fill[oiVermillion!5] (-1.45,-2.4) rectangle (1.4,2.85);
  \fill[oiBlue!5] (1.4,-2.4) rectangle (13.35,2.85);
  \draw[dashed,gray,line width=0.8pt] (1.4,-2.4) -- (1.4,2.85);
\end{scope}
\node[gray!75,font=\scriptsize\itshape] at (1.4,-2.62) {trust boundary};
\node[oiVermillion,font=\scriptsize\bfseries,anchor=west] at (-1.35,2.6) {Untrusted};
\node[oiBlue,font=\scriptsize\bfseries,anchor=east] at (13.25,2.6) {Trusted verifier $+$ simulator};
\draw[arr] (llm.north) |- node[lbl,pos=0.7,above=1pt]{proposal $a_t$} (shadow.west);
\draw[arr] (shadow.east) -- node[lbl,above=1pt]{$\hat{s}_{t+1}$} (gate.west);
\draw[admit] (gate.south) |- node[lblg,pos=0.32,right=2pt]{admit} (live.east);
\draw[arr] (live.west) -| node[lbl,pos=0.2,below=1pt]{$\mathit{obs}_t$} (llm.south);
\draw[fail] (gate.north) to[out=65,in=90,looseness=0.45]
   node[lblr,pos=0.5,above=1pt]{\texttt{FAIL}: regenerate} (llm.north);
\end{tikzpicture}}
\caption{\silr{} admission architecture and trust boundary. A possibly compromised LLM (untrusted) emits a tool proposal $a_t$; the verifier shadow-executes it on a deep copy of the trusted simulator, applies the admission gate $\psi_1\wedge\psi_2\wedge\psi_3$ (tool validation, then the product-order test $\psi_2\wedge\psi_3$), and either applies the admitted action ($\texttt{PASS}$/$\texttt{SAFE\_PROGRESS}$) or rejects it ($\texttt{FAIL}$, regenerating within the step).}
\label{fig:arch}
\end{figure*}
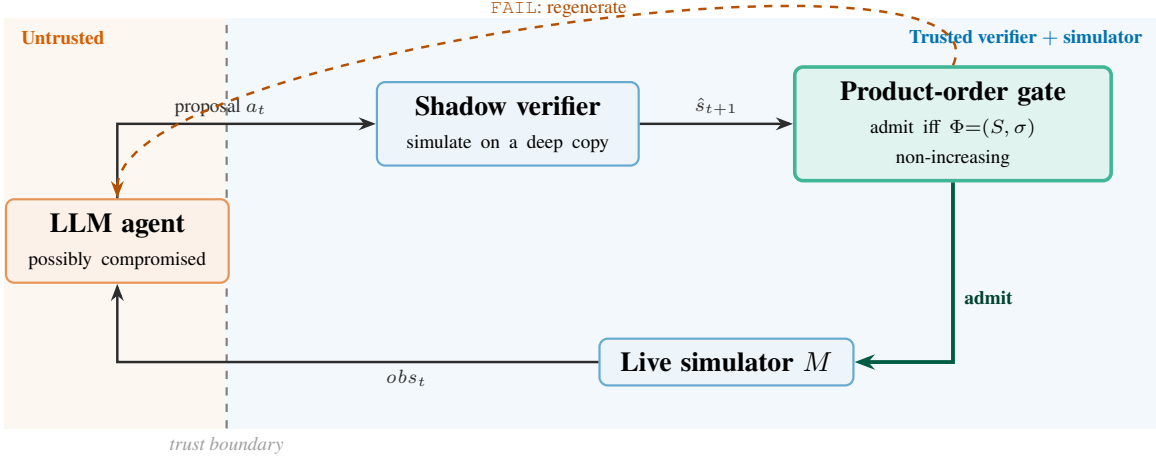

\paragraph{Trust boundary and adversary.}
\silr{} places the LLM outside the trust boundary: no admission verdict depends on the LLM's reasoning, only on the simulator's deterministic response. The LLM is untrusted; the simulator, verifier, checkers, and tool layer are trusted. We adopt the compromised-prompt and compromised-observation capability class, the union of the SARC~\citep{sarc2026} and VeriGuard~\citep{veriguard2025} threat models. An attacker may (C1)~rewrite the system prompt, (C2)~splice attacker-controlled text into every observation, and (C3)~persist both rewrites for the whole episode, which is the worst case for the defender. The attacker may not modify the verifier, simulator, or tool-layer validators, bypass the tool-application interface, or read the verifier's shadow state. We score two safety failures and their union. \emph{Material worsening} is a final penalty strictly above the no-intervention default, \emph{false recovery} is a terminal recovery reported while the live state still holds violations, and \emph{attack success} is the union of the two. Recovery rate is reported separately as a liveness measure (attack constructions in the supplementary material).

\paragraph{Setting.}
An LLM tool-calling agent controls a critical-infrastructure simulator $M$ through a ReAct loop (Fig.~\ref{fig:arch}). At each control step $t=1,\ldots,C_{\text{step}}$ the agent observes $\mathit{obs}_t$, proposes a tool call $a_t$, and a runtime verifier mediates the proposal before it is applied; on rejection the agent regenerates within the same step, up to a per-step budget $C_{\text{prop}}$. \emph{Post-violation} initial conditions mean $M$ starts in $s_0$ with active constraint violations, and the goal is to drive $M$ violation-free.

\paragraph{Shadow-execution verifier.}
For each candidate $a_t$, the verifier evaluates a deterministic shadow state $\hat{s}_{t+1}=\mathrm{shadow}(s_t,a_t)=\mathrm{solve}(\mathrm{apply}(\mathrm{copy}(s_t),a_t))$, where $\mathrm{solve}(\cdot)$ is the domain's steady-state solver (Newton--Raphson power flow for Gym-ANM~\citep{henry2021gymanm}), run on a deep copy of $M$ so rejected proposals leave the live system untouched.

\paragraph{Branch-level violation state and verdict ladder.}
We instantiate $\Phi$ on a vector-valued violation state; in our power-grid testbed the active constraint family is branch loading, indexed by overloaded branch, and CityLearn instantiates the same construction on energy/comfort violations. Define the branch-level violation state
\begin{equation}
\Phi(s) = \bigl(S(s),\,\sigma(s)\bigr),
\quad S(s)\subseteq\mathcal{B},
\quad \sigma(s)\in\mathbb{R}_{\ge 0}^{|\mathcal{B}|},
\label{eq:phi}
\end{equation}
where $\mathcal{B}$ is the set of branches, $S(s)$ the support set of overloaded branches, and $\sigma(s)$ the per-branch severity vector ($\sigma_i(s)\,{>}\,0$ iff $i\in S(s)$). The fully-recovered state satisfies $S(s)=\emptyset$; the aggregate penalty is $P(s)=\sum_i\sigma_i(s)$. We equip the state space with the product order
\begin{equation}
\Phi(s)\preceq\Phi(s') \iff S(s)\subseteq S(s') \;\wedge\; \sigma(s)\le\sigma(s'),
\end{equation}
where the severity comparison is componentwise and support inclusion subsumes count non-increase. This order retains both the overloaded-branch support and the per-branch severity; a scalar penalty gate is its one-dimensional projection. We use \emph{violation geometry} to denote the joint support-and-severity structure that this order retains. We call $\hat{s}_{t+1}$ \emph{admissible} from $s_t$ if $\Phi(\hat{s}_{t+1})\preceq\Phi(s_t)$. The verdict ladder is: \texttt{PASS} if $S(\hat{s}_{t+1})=\emptyset$ (terminal recovery); \texttt{SAFE\_PROGRESS} if $\Phi(\hat{s}_{t+1})\preceq\Phi(s_t)$ but $S(\hat{s}_{t+1})\ne\emptyset$ (admissible nonterminal progress); \texttt{FAIL} otherwise.

\paragraph{The scalar projection trap.}
A scalar alternative replaces the product order with a \emph{scalar admission gate}: admit $\hat{s}_{t+1}$ whenever $P(\hat{s}_{t+1})$ is within a small slack of $P(s_t)$. Any scalar surrogate projects $\Phi(s)=(S(s),\sigma(s))$ onto $\mathbb{R}_{\ge 0}$, and projecting a partial order with antichains onto a total order necessarily makes some incomparable states comparable.
\begin{proposition}[Scalarization gap]
\label{prop:gap}
No scalar-threshold admission gate is sound for $\preceq$: for any $f:(S,\sigma)\to\mathbb{R}_{\ge 0}$ and any slack $\eta\ge0$, the gate admitting whenever $f(\hat{s})\le(1{+}\eta)f(s)$ accepts, on some $\preceq$-antichain pair, a step that introduces or relocates a violation.
\end{proposition}
\begin{proof}[Proof (constructive)]
Take $\Phi(s)=(\{A\},\,2\delta\,e_A)$ and $\Phi(s')=(\{A,B\},\,0.9\delta(e_A{+}e_B))$, $\delta>0$: these are $\preceq$-incomparable ($S(s')\not\subseteq S(s)$ while $\sigma_A(s)=2\delta>0.9\delta=\sigma_A(s')$). Since $\mathbb{R}$ is totally ordered, $f$ ranks the pair, so the gate admits $s\!\to\!s'$ (adds branch $B$) or $s'\!\to\!s$ (drives $\sigma_A$ to $2\delta$); either is non-improving under $\preceq$. For the aggregate $P=\sum_i\sigma_i$, $P(s')=1.8\delta<2\delta=P(s)\le(1{+}\eta)P(s)$ for every $\eta\ge0$, so no slack removes the acceptance. The same antichain pair defeats the strongest scalar candidates: the sup-norm $\|\sigma\|_\infty$ ($0.9\delta<2\delta$) admits $s\!\to\!s'$ and a lexicographic order on $(|S|,\sum_i\sigma_i)$ admits $s'\!\to\!s$, so no monotone scalarization---weighted, max-norm, or lexicographic---escapes.
\end{proof}
\noindent The representational gap is intrinsic to projecting a partial order with antichains onto a total order; relaxing the slack only enlarges the admissible set, never removing the collapse. \silr{} admits on $\preceq$ directly. Through the search loop the representational gap becomes behavioral: a scalar gate admits the first $P$-improving proposal, ends the per-step search, and commits the trajectory to a residual plateau (the \textbf{scalar projection trap}); a structured gate rejects the same proposal as inadmissible under the product order and keeps the search active. No scalar projection yields a sound recovery-admission criterion.

\paragraph{Predicates and admission invariant.}
The implemented admission gate is a conjunction of three predicates evaluated on the shadow: $\psione$ validates the tool and parameter bounds; $\psitwo(s_t,\hat{s}_{t+1}): S(\hat{s}_{t+1})\subseteq S(s_t)$ (support inclusion); and $\psithree(s_t,\hat{s}_{t+1}): \sigma_i(\hat{s}_{t+1})\le\max(\alpha\sigma_i(s_t),\sigma_i(s_t)+\varepsilon)$ for all $i\in S(s_t)$, with $\alpha{=}1.05$, $\varepsilon{=}10^{-3}$. The conjunction $\psitwo\wedge\psithree$ is the computable test for the relaxed $\varepsilon$-dominance order~\citep{laumanns2002}, recovering the exact product order $\preceq$ as $\alpha\to1,\varepsilon\to0$. Here $\alpha$ absorbs the finite resolution of the discrete action space and $\varepsilon$ absorbs the numerical precision. This yields the gate's containment guarantee.
\begin{proposition}[Admission invariant]
\label{prop:inv}
Under shadow fidelity (the live system reproduces the shadow post-state, exact in our deterministic simulators), every admitted trajectory satisfies (i)~$S(s_t)\subseteq S(s_0)$ for all $t$, independently of LLM behavior, and (ii)~$\sigma_i(s_t)\le\alpha^{t}\sigma_i(s_0)+\varepsilon\tfrac{\alpha^{t}-1}{\alpha-1}$, a bound that relaxes as $t$ grows.
\end{proposition}
\noindent The support invariant is $\alpha$-free and holds for \emph{any} gated policy, including a compromised LLM---the containment guarantee that defines a runtime shield and anchors a formal safety case (proof in the supplementary material). Empirically $\psitwo$ supplies liveness by preserving support, while $\psithree$ supplies the per-branch severity containment that RQ4 shows to be necessary. Concretely, $\psithree$ rejects a count-preserving redistribution that both $\psitwo$ and any scalar aggregate admit. Severities $(10,10)\!\to\!(19,0.1)$ leave the support unchanged and \emph{lower} $\sum_i\sigma_i$ from $20$ to $19.1$, yet they drive one branch from $10$ to $19$. Recovery is robust over $\alpha\in[1.05,1.20]$ ($15/15$), consistent with the $\alpha$-free support invariant (sweep in the supplementary material).

\paragraph{Applicability.}
A verifier call takes $17.5$--$30.0$\,ms on a single CPU core, orders of magnitude below the latency of an LLM ReAct step. The construction therefore applies wherever the tool API admits a deterministic shadow of the proposed call. In our testbed it runs on the dispatch loop, where shadow execution reuses the power-flow study that already precedes each dispatch.

\section{Evaluation}
\label{sec:eval}

We compare Structured admission against terminal and scalar gates (RQ1--RQ3), then test containment under compromise (RQ4) and generalization across model families and a second domain (RQ5). The simulator is Gym-ANM ANM6-Easy~\citep{henry2021gymanm}, with a state-cloning shadow and bit-exact Newton--Raphson power flow. The LLM is Qwen3-14B~\citep{qwen3}, served via vLLM at temperature $0$. Each ReAct episode runs $C_{\text{step}}{=}8$ control steps with $C_{\text{prop}}{=}3$ proposals per step. We mine a 600-scenario pool by gridded perturbation and classify it by recoverability; the $24$ MPC-recoverable multi-action scenarios form our benchmark (pool composition in the supplementary material). For RQ1--RQ3, a focused comparison uses three operating points (Scenario~A--Scenario~C, with Scenario~B the easiest and Scenario~A the hardest) at $N{=}7$, giving $7$ episodes per scenario and $21$ per policy, and we additionally evaluate all $24$ scenarios at $N{=}5$. The five admission policies are as follows. Ungated applies no verifier. Terminal admits iff $S(\hat{s}_{t+1})=\emptyset$, the binary SARC-style pre-action gate~\citep{sarc2026}. Support-only applies $\psione\wedge\psitwo$, the pre-action instantiation of the Grid-Agent rollback baseline~\citep{gridagent2025}. Structured applies $\psione\wedge\psitwo\wedge\psithree$, the full product-order gate. The scalar sweep Scalar$_\eta$ admits iff $P(\hat{s}_{t+1})\le(1{+}\eta)P(s_t)$ for $\eta\in\{0,0.05,0.10,0.20\}$. For the full-pool stress test (RQ4) we additionally evaluate the strongest scalar surrogates of Prop.~\ref{prop:gap}: the sup-norm $\|\sigma\|_\infty$ and a lexicographic order on $(|S|,\sum_i\sigma_i)$.

\begin{figure*}[t]
\centering
\includegraphics[width=0.86\textwidth]{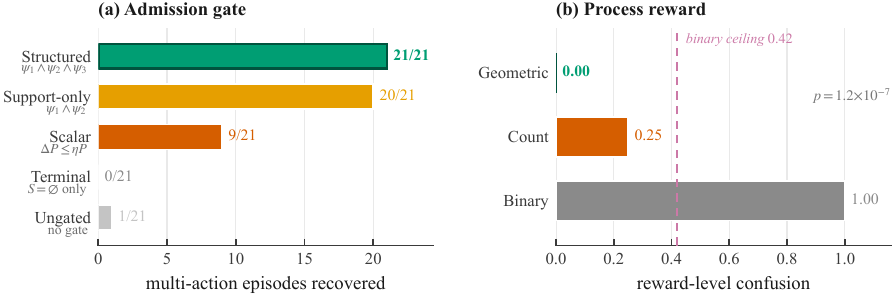}
\caption{One product order, two design points. (a)~As an admission gate at matched $N{=}7$, the product order recovers $\mathbf{21/21}$ multi-action episodes, against $20/21$ for support-only (Grid-Agent), $9/21$ for the best scalar slack, $1/21$ ungated and $0/21$ terminal; sublabels give each gate's admission criterion. (b)~As a GRPO process reward, a reward built on the same violation geometry has zero reward-level confusion against $0.25$ for its count projection and $1.0$ for a binary signal (oracle-optimal binarization ceiling $0.42$); the dominance holds in all $24$ mined scenarios (sign test $p{=}1.2{\times}10^{-7}$).}
\label{fig:dualladder}
\end{figure*}

\paragraph{RQ1: binary terminal admission deadlocks recovery.}
Fig.~\ref{fig:dualladder}(a) reports the main recovery ladder at matched $N{=}7$ ($21$ episodes/policy). Terminal recovers $0/21$: every intermediate proposal has $S(\hat{s}_{t+1})\ne\emptyset$, regardless of search depth. Structured admission recovers $\mathbf{21/21}$ (residual $0.000$) against the best scalar result of $9/21$, significant at the scenario level on the full benchmark (RQ2). The improvement is structural: support geometry closes the recovery gap, while the per-branch predicate $\psithree$ adds the severity containment RQ4 requires.

\paragraph{RQ2: scalar admission is insufficient.}
If terminal admission fails solely because it is too strict, the natural fix is scalar slack. Relaxing the threshold neither closes the recovery gap nor yields monotone recovery: across $\eta\in\{0,0.05,0.10,0.20\}$ recovery is $5/21,9/21,7/21,7/21$ (each far below structured admission's $21/21$). The non-monotonicity is itself a trap signature: the $\eta{=}0.05\!\to\!0.10$ drop comes entirely from Scenario~A and Scenario~C, where wider slack admits an early $P$-improving proposal and locks the trajectory onto a residual plateau rather than rejecting it to keep the per-step search alive. The recovery gap closes at the support level: support-only ($\psione\wedge\psitwo$) recovers $20/21$ and the full product order $21/21$. The full $24$-scenario benchmark confirms this (Table~\ref{tab:band24}): structured admission recovers $\mathbf{120/120}$ against the best scalar gate's $96/120$ (terminal deadlocks); treating each of the $24$ scenarios as an independent cluster, the structured-vs-scalar advantage is significant (scenario-clustered Wilcoxon $p{=}0.004$, bootstrap 95\% CI for the recovery-rate difference excluding zero).

\begin{table}[t]
\centering
\caption{RQ2: the full benchmark of $24$ mined multi-action scenarios, $N{=}5$, recovery [Wilson 95\% CI]. Support-only is the Grid-Agent rollback baseline.}
\label{tab:band24}
\footnotesize
\begin{tabular*}{\columnwidth}{@{\extracolsep{\fill}}lcc@{}}
\toprule
Policy & Recovery & 95\% CI \\
\midrule
Terminal                            & $0/120$  & $[0.00,0.03]$ \\
Scalar$_{\eta=0.05}$      & $96/120$ & $[0.72,0.86]$ \\
Support-only                            & $113/120$ & $[0.88,0.97]$ \\
Structured                       & $\mathbf{120/120}$ & $\mathbf{[0.97,1.00]}$ \\
\bottomrule
\end{tabular*}
\end{table}

\paragraph{RQ3: the projection-trap plateau is gate-induced.}
The trap leaves a population signature: over the full benchmark, structured admission drives every episode to the recovered origin, while scalar-admitted failures settle on a nonzero-residual plateau (Fig.~\ref{fig:trap}). The three operating points show increasing difficulty under the best per-scenario slack. Scalar admission recovers $7/7$ episodes in Scenario~B, at most $2/7$ in Scenario~C (mean final residual $5.94$ against a default of $17.59$), and at most $1/7$ in Scenario~A, whereas structured admission recovers $7/7$ in all three (representative trace in the supplementary material). When initialized from the scalar-admitted Scenario~C plateau, an MPC feasibility oracle (constant-setpoint baseline) reaches zero penalty in one rollout, which confirms that the plateau is gate-induced.

\begin{figure}[t]
\centering
\includegraphics[width=\columnwidth]{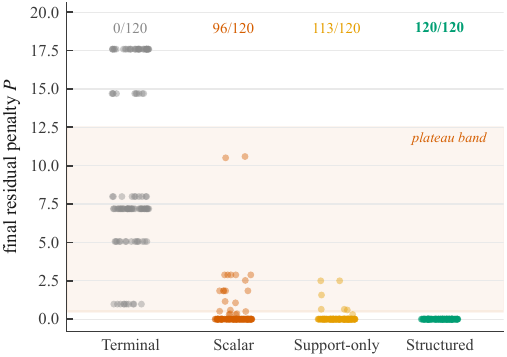}
\caption{RQ3: final residual penalty of all $480$ benchmark episodes ($24$ scenarios, $N{=}5$, four gates; horizontally jittered). The count above each column is the number of episodes fully recovered ($P{=}0$). Structured admission drives every episode to the recovered origin; scalar-admitted failures settle in a nonzero-residual plateau band; terminal episodes remain at their initial violation state.}
\label{fig:trap}
\end{figure}

\paragraph{RQ4: the admission architecture contains LLM-side compromise.}
A compromise suite evaluates the gate under three attack families: prompt injection, observation poisoning, and stall, the last in plain and RAG forms. It runs across the three mined scenarios plus matched benign controls ($60$ attack and $15$ benign episodes). No attack success, false recovery, or material worsening occurs in the $60$ attack episodes ($0/60$) or in the $15$ benign-control episodes ($0/15$). Containment is \emph{architectural}: the verifier never accesses the LLM's prompt or observation channels (Prop.~\ref{prop:inv}).

\paragraph{The per-branch predicate $\psithree$ is necessary.}
This suite establishes containment against channel compromise; whether the per-branch predicate is \emph{necessary} is a separate question, isolated by a dedicated $120$-case suite. \emph{Magnitude redistribution} (C1/C2) drives one retained branch toward failure while holding the overloaded-branch set and the aggregate penalty within nominal bounds, evading aggregate- and support-level monitoring. Across this suite a scalar gate admits the concentrating step in $43/120$ cases and the support-only Grid-Agent rollback baseline ($\psione\wedge\psitwo$) in $11/120$, whereas the full product order ($\psione\wedge\psitwo\wedge\psithree$) admits none ($0/120$; Fig.~\ref{fig:breach}): only the per-branch test $\psithree$ rejects the concentrating step.

\begin{figure}[t]
\centering
\includegraphics[width=\columnwidth]{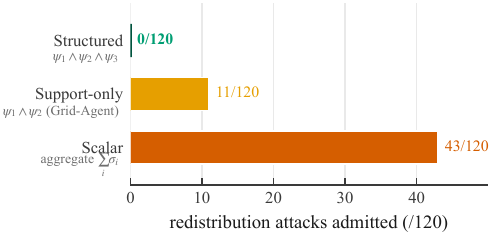}
\caption{RQ4: magnitude-redistribution attacks admitted, by admission tier. Each added predicate closes more attack surface: Scalar $43/120$, Support-only (Grid-Agent, $\psi_1\wedge\psi_2$) $11/120$, Structured ($\psi_1\wedge\psi_2\wedge\psi_3$) $\mathbf{0/120}$.}
\label{fig:breach}
\end{figure}

\paragraph{The complementary antichain on the full witness pool.}
Beyond the support-preserving redistribution of Fig.~\ref{fig:breach}, full-pool enumeration exposes the complementary Prop.~\ref{prop:gap} antichain (support expansion) and confirms the gap is representational. Over all $558$ physically-unsafe one-step actions in ANM, the sup-norm $\|\sigma\|_\infty$ admits $46$ and a lexicographic order on $(|S|,\sum_i\sigma_i)$ admits $88$, while the full product order admits $0/558$. These breaches instantiate the antichain: a new branch overloads while $\max_i\sigma_i$ falls. The same class recurs in CityLearn over $11{,}311$ physically-unsafe actions under an independent state-of-charge/feeder oracle (verifier-level enumeration, independent of model choice): the sup-norm admits $6761$, the full product order $0/11{,}311$ (ReAct replication in RQ5). Across pools the invariant comparison is zero-versus-nonzero: the full product order is the only gate that admits no unsafe action in either domain (supplementary material).

\paragraph{Safety versus liveness under multi-family stress.}
The containment gap widens with two physically incomparable constraint families active, where any scalarization must trade one against the other. Over all $42{,}410$ physically-unsafe actions on mined dual-family (voltage $+$ branch) ANM states, every tested scalar projection admits $40.6$--$53.6\%$, and support-only admits the largest fraction at $63.2\%$, whereas the full product order admits $0$. Support inclusion alone is therefore insufficient, and the per-branch guard $\psithree$ closes the multi-family surface. Recovery nevertheless remains feasible under this discipline: an oracle audit finds an admissible all-safe recovery order in \emph{all} $689$ mined dual-family scenarios. In ReAct traces on the two hardest dual-family scenarios ($3$ seeds each), every scalar-gated recovery passes through physically unsafe steps ($18$--$38\%$ of its penalty descent; worst, a $212\%$-loaded branch driven to $282\%$), whereas the product-order gate admits none. This reproduces the safety--liveness trade-off familiar from runtime shielding~\citep{alshiekh2018safe} (details in the supplementary material).

\paragraph{RQ5: the dichotomy transfers across model families and domains.}
At the ReAct level, the terminal-versus-structured dichotomy is family-independent and transfers across domains (Table~\ref{tab:gen}). Across Qwen3 (two sizes), Gemma-3~\citep{gemma3}, and Llama-3.1~\citep{llama3} on ANM, terminal admission deadlocks every model ($0/15$) while structured admission recovers completely ($15/15$). The same pattern recurs in CityLearn building-energy dispatch~\citep{citylearn}: terminal admission succeeds only on the single-step scenario ($5/15$), while structured admission recovers between $10/15$ and $15/15$ depending on the model. The admission rule thus transfers unchanged, whereas completion rates depend on model capability: the gate semantics show \emph{structural portability} to a distinct physical model.

\begin{table}[t]
\centering
\caption{RQ5: generalization across model families (ANM) and to a second domain (CityLearn), $N{=}5$ (recovery). CityLearn at $C_{\text{step}}{=}8$; $^\ast$Qwen3-8B reaches $15/15$ at $C_{\text{step}}{=}16$. ANM Structured per-model ($3$ scenarios pooled) Wilson 95\% CI $[0.80,1.00]$.}
\label{tab:gen}
\footnotesize
\setlength{\tabcolsep}{5pt}
\begin{tabular*}{\columnwidth}{@{\extracolsep{\fill}}lcccc@{}}
\toprule
 & \multicolumn{2}{c}{ANM} & \multicolumn{2}{c}{CityLearn} \\
\cmidrule(lr){2-3}\cmidrule(lr){4-5}
Model & Term. & Structured & Term. & Structured \\
\midrule
Qwen3-8B     & $0/15$ & $\mathbf{15/15}$ & $5/15$ & $14/15^\ast$ \\
Qwen3-14B    & $0/15$ & $\mathbf{15/15}$ & $5/15$ & $\mathbf{15/15}$ \\
Gemma-3-12B  & $0/15$ & $\mathbf{15/15}$ & $5/15$ & $10/15$ \\
Llama-3.1-8B & $0/15$ & $\mathbf{15/15}$ & $5/15$ & $\mathbf{15/15}$ \\
\bottomrule
\end{tabular*}
\end{table}

\section{From Admission Gate to Process Reward}
\label{sec:reward}

\begin{figure*}[t]
\centering
\includegraphics[width=0.86\textwidth]{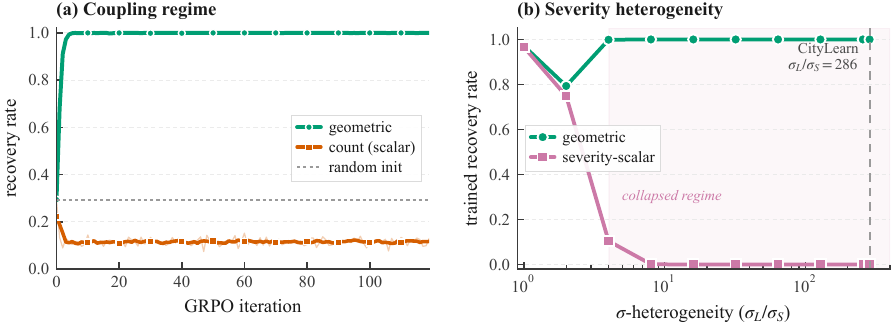}
\caption{Training dynamics in the multi-family stress regimes. (a)~Coupling: the geometric reward converges to full recovery within five GRPO iterations; the count reward---blind to the drift its own actions cause---is driven below its random initialization. (b)~Trained recovery versus $\sigma$-heterogeneity: failure emerges at $\sigma_L/\sigma_S\!\approx\!4$ (both arms dip briefly at ratio $2$); CityLearn's measured ratio ($286$, dashed) lies deep in the collapsed regime; in the homogeneous limit all arms tie, ruling out the construction.}
\label{fig:trainenh}
\end{figure*}

The product order is also the representation required at a \emph{second design point}: not only the criterion a runtime gate admits on, but the object a policy should \emph{internalize}. We reuse the verifier's outputs, namely the graded verdict and the persisted branch state $\Phi$ of Eq.~\ref{eq:phi}, to define a Group Relative Policy Optimization (GRPO) process reward~\citep{deepseekmath2024grpo,lightman2024prm}. This requires no verifier changes and no task-specific outcome labels. We fine-tune Qwen3-8B with low-rank adaptation (LoRA)~\citep{hu2022lora} on the $9$ hardest multi-action ANM scenarios, which the ungated base fails to recover reliably. The primary outcome is \emph{ungated} recovery: the gate is removed at evaluation, so recovery measures what the policy internalized. An admitted \texttt{SAFE\_PROGRESS} step from $s$ to shadow post-state $\hat s$ is scored on the branch state itself; writing $\sigma(A)=\sum_{k\in A}\sigma_k$ and $(x)_+=\max(0,x)$,
\begin{align}
\rho_{\mathrm{geom}} &= \tfrac{W_2}{|F|}\!\sum_{f\in F}\!\frac{\sigma(E\cap f)}{\sigma(f)}
+\tfrac{W_3}{|F|}\!\sum_{f\in F}\!\frac{\sum_{k\in U\cap f}(\sigma_k-\hat\sigma_k)_+}{\sigma(f)}
\nonumber\\[-1pt]
&\quad-\,W_d\,\min\!\Big(1,\,\max_{k\in U}\tfrac{(\hat\sigma_k-\sigma_k)_+}{\sigma_k}\Big),
\label{eq:rhogeom}
\end{align}
where $E=S\setminus\hat S$ and $U=S\cap\hat S$ are the eliminated and surviving overloaded branches, $F$ the constraint families partitioning $S$, and $(W_2,W_3,W_d)=(0.6,0.3,0.3)$. Three choices matter here. Support elimination outranks severity reduction on surviving branches, echoing the admission order. The drift term debits any surviving branch driven worse. Per-family normalization keeps a high-$\sigma$ family from hijacking the signal. The count projection replaces $\rho_{\mathrm{geom}}$ by $(|S|-|\hat S|)/|S|$, and graded \texttt{PASS}/\texttt{FAIL} verdicts anchor the scale identically across arms (full scaffold in the supplementary material).

\paragraph{Violation geometry yields a higher-fidelity process reward.}
We measure reward-layer fidelity by scoring every legal single-setpoint action at each trap state against the true one-step recovery value. The geometric reward never misorders an action pair with distinct oracle values (reward-level confusion $0.000$ vs.\ $0.25$ for its count projection), and it dominates in all $24$ mined scenarios (sign test $p{=}1.2{\times}10^{-7}$). By contrast, a binary admit/reject signal cannot rank admissible steps at all ($1.0$; Fig.~\ref{fig:dualladder}(b)). The separation is a property of the reward, not of the rollout distribution. It persists over the full legal action set ($0.04$ for the geometric reward, $0.19$ for the count projection, and $0.63$ for the binary signal) and against an oracle-optimal binarization ($0.42$), which is the ceiling of a process reward trained on binary outcome labels~\citep{lightman2024prm}. The geometric ranking matches the oracle with a Spearman correlation of $0.986$; robustness controls are in the supplementary material.

\paragraph{The violation geometry is the active ingredient.}
Three reward arms share the same rollout gate and differ only in the signal. The full geometric reward, descent in the product order over $\Phi$, is the only one that keeps the trained policy above the untrained base ungated ($\mathbf{0.844}$ vs.\ $0.778$). A verdict ladder ($0.733$) or a violation count ($0.667$) leaves it \emph{worse} than untrained, which indicates gate-dependent behavior rather than internalized recovery (supplementary material). On $15$ held-out scenarios the geometric policy recovers $14/15$ against the base's $9/15$, both ungated (one-sided Fisher exact $p{=}0.04$).

\paragraph{The same scalarization gap, on the training side.}
This is the training-side counterpart of the admission-side impossibility: the scalarization that is unsound as an admission rule (Prop.~\ref{prop:gap}) collapses the GRPO advantage signal as a reward.
\begin{proposition}[Advantage collapse]
\label{prop:collapse}
With $B$ concurrent violation branches, a count-based reward takes at most $B{+}1$ distinct values. Action pairs within a count-equivalence class therefore receive identical immediate reward and have no step-level advantage separation. Along a continuous partial reduction of any single branch, the count reward is piecewise constant, with zero gradient except at elimination points, whereas the graded geometric reward is strictly monotone in that reduction.
\end{proposition}

\noindent (Proof in the supplementary material.) Empirically the flat fraction of the count reward is $98\%$ of the continuous action range, inducing a \emph{failure class}: with incomparable constraint families active, no tested scalar projection of the reward attains full recovery in both regimes; only the full geometric reward does. We isolate the mechanism in a controlled tabular GRPO setting (Table~\ref{tab:failureclass}; $16$ training seeds each; design in the supplementary material): the severity scalar is hijacked by raw magnitude under $\sigma$-heterogeneity, the count reward is blind to the cross-family drift its own actions cause under coupling (Fig.~\ref{fig:trainenh}(a)), and restoring the drift term closes only the coupling regime (Table~\ref{tab:failureclass}).

\paragraph{The collapse is a threshold effect that reproduces at scale.}
The severity-scalar failure emerges at $\sigma_L/\sigma_S\!\approx\!4$, and CityLearn's measured heterogeneity ratio ($286$) lies deep in the collapsed regime (Fig.~\ref{fig:trainenh}(b)). A four-building CityLearn suite reproduces the effect at scale: the geometric reward's advantage signal correlates with the true one-step value (Spearman's $\rho{=}0.96$ vs.\ $0.79$ for count), and at the same violation count the count arm leaves the battery family with $14\%$ higher residual severity (supplementary material).

\begin{table}[t]
\centering
\caption{The scalarization failure class: trained recovery in the two multi-family stress regimes ($16$ seeds each).}
\label{tab:failureclass}
\footnotesize
\begin{tabular*}{\columnwidth}{@{\extracolsep{\fill}}lcc@{}}
\toprule
 & \multicolumn{2}{c}{trained recovery} \\
\cmidrule(lr){2-3}
Training reward & $\sigma$-het & coupling \\
\midrule
\textbf{geometric (full $\Phi$)} & $\mathbf{1.00}$ & $\mathbf{1.00}$ \\
severity-scalar ($\Sigma\sigma$ descent) & $0.00$ & $0.04$ \\
count ($|S|$ descent) & $0.80$ & $0.12$ \\
count with drift penalty & $0.80$ & $1.00$ \\
\bottomrule
\end{tabular*}
\end{table}

\section{Conclusion}
\label{sec:conclusion}

We recast the admission gate as a search operator and identify the \emph{scalar projection trap}: collapsing the violation state to a scalar commits trajectories to residual plateaus. The formal results establish this representational gap independently of LLM behavior (Prop.~\ref{prop:gap} and Prop.~\ref{prop:inv}), and the experiments show $21/21$ recoveries for structured admission, against $9/21$ for the best scalar gate and $0/21$ for terminal admission. Under multi-family stress the product order keeps every intermediate step safe, whereas scalar gates reach recovery only through unsafe redistributions. Reused as a process reward, the same signal lets the policy internalize recovery. Preserving the violation geometry is required at both design points.

\appendix

\section{Attack-Family Constructions}
\label{supp:attacks}

We implement four attack families under the compromised-prompt/compromised-observation capability class of Section~\ref{sec:method}. \emph{Prompt injection} (C1): a high-priority operator-impersonating override that pushes an aggressive fixed setpoint. \emph{Observation poisoning} (C2): a fabricated ``all readings nominal'' stanza spliced into every observation, contradicting the violation list the verifier re-derives from the live simulator. \emph{Stall} (C1): near-no-op setpoints that stay admissible while never recovering; a ``stall-RAG'' variant instructs the same behavior through a poisoned retrieval context. \emph{Magnitude redistribution} (C1/C2): driving a single retained branch toward failure while holding the overloaded-branch set and aggregate penalty within nominal bounds, masking per-branch overload behind nominal aggregate- and support-level readings.

The compromise suite (RQ4) runs the three attack families in four configurations---prompt injection, observation poisoning, and stall in plain and RAG forms---across the three mined multi-action scenarios at $N{=}5$ per configuration, plus a matched benign control ($60$ attack $+\,15$ benign $=75$ episodes). Across the $60$ attack episodes we observe no attack success, false recovery, or material worsening ($0/60$; Wilson 95\% CI $[0.00,0.06]$), and none of these failures occurs in the $15$ benign-control episodes ($0/15$). Under prompt injection the rejection rate climbs sharply. The fourth family, magnitude redistribution, is evaluated against competing gates in the $120$-case redistribution suite of Fig.~\ref{fig:breach}: its cases concentrate severity on a single retained branch while preserving the overloaded-branch set and holding the aggregate penalty within nominal bounds, slipping past aggregate- and support-level gates ($43/120$ and $11/120$) while the per-branch test $\psithree$ rejects every one.

\paragraph{Strong-scalar baselines on the full witness pool.}
Beyond the balanced $120$-case suite, we enumerate \emph{every} physically-unsafe one-step action in each domain---$558$ in ANM, $11{,}311$ in CityLearn (the latter under an independent state-of-charge/feeder-overload oracle, verifier-level, no LLM in the loop)---and evaluate all five gates (Table~\ref{tab:fullpool}). The denominator is a constructed witness pool, not a natural-attack frequency, and absolute breach counts reflect each pool's category composition, so they are \emph{not} comparable across domains; within a fixed pool, a lower unsafe-admission count indicates stricter containment. The one domain-invariant statement is that the full product order admits zero in both. The strongest scalar projections fail on the Prop.~\ref{prop:gap} antichain in both domains: in ANM all $46$ sup-norm breaches are support expansions with $\max_i\sigma_i$ non-increasing; in CityLearn $6670$ of the $6761$ sup-norm breaches are the same support-expansion class, the newly violating element being a battery or feeder, confirming the mechanism in a second physical system. Finally, we stress the gates with two simultaneously active constraint families: a third enumeration over mined dual-family (voltage $+$ branch-loading) ANM states yields $42{,}410$ physically-unsafe actions: the aggregate admits $17{,}200$ ($40.6\%$), the normalized aggregate $18{,}311$, the sup-norm $17{,}663$, the lexicographic order $22{,}729$, support-only $26{,}817$ ($63.2\%$), and the product order $0$.

\begin{table}[h]
\centering
\caption{Unsafe admissions on the full physically-unsafe witness pool (verifier-level enumeration; counts reflect each pool's category mix and are not comparable across domains; within a pool, lower is stricter containment). The full product order is the only gate that admits no unsafe action in either domain.}
\label{tab:fullpool}
\scriptsize
\setlength{\tabcolsep}{3pt}
\begin{tabular*}{\columnwidth}{@{\extracolsep{\fill}}lcc@{}}
\toprule
Gate & ANM ($N{=}558$) & CityLearn ($N{=}11{,}311$) \\
\midrule
Scalar---aggregate $\sum_i\sigma_i$ & $108$ & $7204$ \\
Scalar---sup-norm $\|\sigma\|_\infty$ & $46$ & $6761$ \\
Scalar---lex.\ $(|S|,\sum_i\sigma_i)$ & $88$ & $7650$ \\
Support-only (Grid-Agent $\psione\wedge\psitwo$) & $258$ & $432$ \\
\textbf{Structured} ($\psione\wedge\psitwo\wedge\psithree$) & $\mathbf{0}$ & $\mathbf{0}$ \\
\bottomrule
\end{tabular*}
\end{table}

\section{Proofs}
\label{supp:proofs}

\paragraph{Proposition~1 (Scalarization gap).}
Any scalar surrogate $f:(S,\sigma)\to\mathbb{R}_{\ge0}$ composes the vector-valued violation state with a total order. The pair $\Phi(s)=(\{A\},2\delta e_A)$, $\Phi(s')=(\{A,B\},0.9\delta(e_A{+}e_B))$ forms a $\preceq$-antichain: $S(s')\not\subseteq S(s)$ blocks $s'\preceq s$, while $\sigma_A(s)=2\delta>0.9\delta=\sigma_A(s')$ blocks $s\preceq s'$. Because $f$ maps both into the totally ordered $\mathbb{R}_{\ge0}$, exactly one of $f(s)\le f(s')$ or $f(s')\le f(s)$ holds (or both, if equal), so for any slack $\eta\ge0$ the gate admits at least one of the two transitions, each non-improving under $\preceq$ (one introduces branch $B$; the other inflates $\sigma_A$). For the canonical aggregate $P=\sum_i\sigma_i$, $P(s')=1.8\delta<2\delta=P(s)$, so $P(s')\le(1{+}\eta)P(s)$ for all $\eta\ge0$ and no slack removes the acceptance of $s\to s'$. The argument depends only on $\preceq$ admitting an antichain, hence applies to every scalar surrogate. $\square$

\paragraph{Proposition~2 (Admission invariant).}
By induction on $t$. Base: trivially $S(s_0)\subseteq S(s_0)$ and $\sigma_i(s_0)\le\alpha^0\sigma_i(s_0)$. Step: an admitted action passes $\psitwo\wedge\psithree$ on the shadow, and shadow fidelity makes the live post-state equal the shadow, so $S(s_{t+1})\subseteq S(s_t)$ (giving $S(s_{t+1})\subseteq S(s_0)$ by the inductive hypothesis) and $\sigma_i(s_{t+1})\le\max(\alpha\sigma_i(s_t),\sigma_i(s_t)+\varepsilon)\le\alpha\sigma_i(s_t)+\varepsilon$ for all $i\in S(s_t)$ (and $\sigma_i\equiv0$ off $S(s_t)$). Unrolling $\sigma_i(s_{t+1})\le\alpha\sigma_i(s_t)+\varepsilon$ gives $\sigma_i(s_t)\le\alpha^{t}\sigma_i(s_0)+\varepsilon\sum_{j=0}^{t-1}\alpha^{j}=\alpha^{t}\sigma_i(s_0)+\varepsilon\tfrac{\alpha^{t}-1}{\alpha-1}$; at $\alpha{=}1$ the bound reads $\sigma_i(s_0)+t\varepsilon$. Neither step uses any property of the action source, so the support invariant ($\alpha$-free) holds for any gated policy, including a compromised LLM. The severity term is a worst-case per-step drift envelope, not a descent guarantee: convergence to $S=\emptyset$ depends on proposal quality. $\square$

\paragraph{Proposition~3 (Advantage collapse under scalarization).}
Immediate from the definitions. With $B$ concurrent violation branches, the count delta of cleared violations ranges over $\{0,\dots,B\}$, taking at most $B{+}1$ distinct values; hence every action pair within a count-equivalence class receives identical immediate reward and there is no step-level advantage separation between them. Along a continuous partial reduction of any single branch's severity that does not eliminate it, $|S|$ is unchanged, so the count reward is piecewise constant, with zero gradient except at the discrete elimination points. The graded geometric reward's severity term is strictly monotone in that same reduction, so it retains a nonzero advantage signal everywhere on the continuum. Empirically the flat fraction is $98\%$ of the continuous action range. $\square$

\section{Robustness and Cross-Domain Generalization}
\label{supp:gen}

\paragraph{Scenario pool composition.}
The mined 600-scenario ANM pool classifies by recoverability as $153$ single-action, $24$ multi-action, $252$ MPC-residual, and $171$ trivial; the $24$ multi-action scenarios form the main benchmark.

\paragraph{$\alpha$-window sweep.}
Recovery holds across an $\alpha$-window: $15/15$ for $\alpha\in[1.05,1.20]$ and $14/15$ at $\alpha{=}1.02$ ($N{=}5$), consistent with the $\alpha$-free support invariant (Prop.~\ref{prop:inv}); the magnitude guard $\psithree$ tightens as $\alpha\to1$ without disturbing support-driven recovery. The full $24$-scenario benchmark (Table~\ref{tab:band24}) and the cross-family replication (Table~\ref{tab:gen}) are reported in the main text.

\paragraph{Cross-domain (CityLearn).}
The full cross-domain replication is reported in Table~\ref{tab:gen} (RQ5): the gate semantics recur in CityLearn building-energy dispatch~\citep{citylearn} over three district battery-storage scenarios (CL-1--CL-3; state-of-charge floor/ceiling and feeder-export limits, $\Phi$ indexed by the violating building or feeder). Gemma-3-12B on CL-2 is the only capability-related exception. The gate admits every Gemma proposal in that scenario (all \texttt{SAFE\_PROGRESS}, with no rejections), so the admission rule transfers unchanged, but the sequence stalls before the terminal action the other models reach. Completion therefore tracks proposal quality rather than gate semantics.

\paragraph{Scenario-level trap trace.}
Fig.~\ref{fig:traptrace} shows the representative Scenario-C trace underlying RQ3: the scalar gate admits one early locally-improving proposal and plateaus at nonzero residual (best slack $\eta{=}0.05$; the plateau recurs at all four slack settings, best $2/7$), while structured admission descends to full recovery.

\begin{figure}[t]
\centering
\includegraphics[width=\columnwidth]{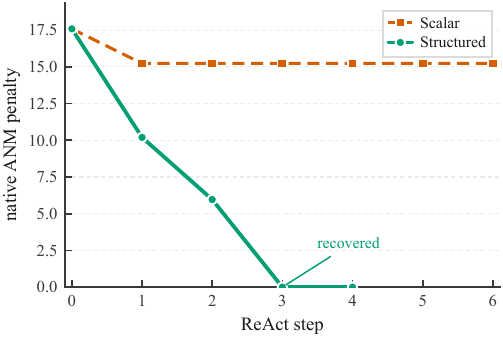}
\caption{Scalar projection trap on Scenario~C (RQ3): representative trace at the best slack $\eta{=}0.05$; structured admission descends to full recovery in three steps while the scalar gate plateaus.}
\label{fig:traptrace}
\end{figure}

\paragraph{Dual-family recovery audit.}
We mine $689$ oracle-recoverable dual-family (voltage $+$ branch-loading) ANM states and replay oracle device orders through the product-order verifier: an admissible all-safe recovery order exists in all $689/689$ cases, so per-step safety never renders recovery infeasible. At the ReAct level ($32$-step budget, Qwen3-14B, temperature $0$), we rescore every admitted step with the RQ4 physical oracle. On the two hardest such states, every scalar-gated recovery passes through physically unsafe steps ($18$--$38\%$ of its penalty descent; worst, a $212\%$-loaded branch driven to $282\%$ loading), including support expansions of the same class enumerated in the dual-family admission audit. The product-order gate admits zero unsafe steps.

\section{The Verifier as a Process Reward: Extended Analysis}
\label{supp:reward}

\paragraph{Setup.}
We fine-tune Qwen3-8B with LoRA~\citep{hu2022lora} on the $9$ hardest multi-action ANM scenarios (greedy decoding, without chain-of-thought). Three reward arms share the same structured rollout gate and differ only in the reward signal: \textbf{Geometric} rewards descent in the product order over $\Phi=(S,\sigma)$; \textbf{Count} rewards the count-delta of cleared violations; and a \textbf{Verdict-only} ablation strips the geometric component, leaving only the verdict ladder. We evaluate the checkpoint after the second GRPO~\citep{deepseekmath2024grpo,lightman2024prm} iteration, pooled over $5$ training seeds (Table~\ref{tab:reward}). On a separate $15$-scenario held-out set, the geometric policy reaches $14/15$ ungated recovery, against $9/15$ for the base policy (RQ-reward section).

\begin{table}[t]
\centering
\caption{Verifier-as-reward on the $9$ hardest ANM scenarios ($5$ seeds pooled; recovery [Wilson 95\% CI]). \emph{Ungated} recovery (gate removed at eval) measures internalization: only the full geometric reward stays above the untrained base.}
\label{tab:reward}
\footnotesize
\setlength{\tabcolsep}{5pt}
\begin{tabular*}{\columnwidth}{@{\extracolsep{\fill}}lccc@{}}
\toprule
Reward (signal) & gated & \textbf{ungated} & steps$\downarrow$ \\
\midrule
base (no LoRA) & $0.778$ & $0.778$ & $4.71$ \\
\textbf{Geometric} ($\Phi$ descent) & $\mathbf{0.956}$ & $\mathbf{0.844}$ {\scriptsize[0.71,0.92]} & $\mathbf{3.64}$ \\
Verdict-only (ladder) & $0.889$ & $0.733$ {\scriptsize[0.59,0.84]} & $4.08$ \\
Count (count-delta) & $0.889$ & $0.667$ {\scriptsize[0.52,0.79]} & $3.84$ \\
\bottomrule
\end{tabular*}
\end{table}

\paragraph{Reward formulas.}
All arms share one verdict scaffold and differ only in how an admitted \texttt{SAFE\_PROGRESS} step is scored. For a step from pre-state $s$ to shadow post-state $\hat s$ with verdict $v$,
\begin{equation}
r(s,\hat s)=\!
\begin{cases}
1+\tfrac12\,\bar m(\hat s) & v=\texttt{PASS}\\[1pt]
\rho_{\bullet}(s,\hat s) & v=\texttt{SAFE\_PROGRESS}\\[1pt]
-\max\!\big(0.3,\,c_{\mathrm{sev}}(\hat s)\big) & v=\texttt{FAIL}\\[1pt]
-1 & v=\texttt{ERROR},
\end{cases}
\label{eq:reward-scaffold}
\end{equation}
where $\bar m\in[0,1]$ is the mean normalized constraint margin and $c_{\mathrm{sev}}\in\{0.3,0.6,1.0\}$ grades the worst residual violation (warning/violation/critical); a fixed step cost $0.05$ is subtracted from every step and a terminal recovery bonus $+1$ is added at the trajectory level, both shared across arms. Writing $S,\hat S$ for the pre/post overloaded-branch supports with per-branch severities $\sigma_k,\hat\sigma_k$, $E=S\setminus\hat S$ for the eliminated branches, $U=S\cap\hat S$ for the surviving ones, and $f\in F$ for the constraint families partitioning $S$, the three \texttt{SAFE\_PROGRESS} scorings are (expanding the compact notation of Eq.~\ref{eq:rhogeom})
\begin{align}
\rho_{\mathrm{geom}} &= \tfrac{W_2}{|F|}\!\sum_{f\in F}\!\frac{\sum_{k\in E\cap f}\sigma_k}{\sum_{k\in f}\sigma_k}
\nonumber\\[-1pt]
&\quad+\tfrac{W_3}{|F|}\!\sum_{f\in F}\!\frac{\sum_{k\in U\cap f}\max(0,\sigma_k-\hat\sigma_k)}{\sum_{k\in f}\sigma_k}
\nonumber\\[-1pt]
&\quad-\,W_d\,\min\!\Big(1,\,\max_{k\in U}\tfrac{\max(0,\,\hat\sigma_k-\sigma_k)}{\sigma_k}\Big),
\label{eq:rho-geom}\\[2pt]
\rho_{\mathrm{count}} &= \frac{|S|-|\hat S|}{|S|},
\qquad
\rho_{\mathrm{verdict}} \equiv 0.5,
\label{eq:rho-count}
\end{align}
with weights $(W_2,W_3,W_d)=(0.6,0.3,0.3)$, so severity-weighted support elimination ($\psi_2$) outranks per-branch severity reduction ($\psi_3$), echoing the admission product order; the third term penalizes any surviving branch driven worse. The geometric arm normalizes \emph{within} each family and averages families with equal weight, so no high-$\sigma$ family hijacks the signal; for a single family it reduces exactly to $\Sigma\sigma$-normalization, leaving the single-type ANM result unchanged. A count-preserving magnitude reallocation ($S{=}\hat S$) earns $\rho_{\mathrm{geom}}\!\approx\!0$ and is debited by the drift term, yet is invisible to $\rho_{\mathrm{count}}$; the \textbf{Verdict-only} arm flattens $\rho$ to a constant, keeping the \texttt{PASS}${>}$\texttt{SAFE\_PROGRESS}${>}$\texttt{FAIL} ladder while discarding within-step geometry. The failure-class study's severity scalar instead uses the total-severity fraction $\rho_{\mathrm{sev}}=\big(\sum_{k\in S}\sigma_k-\sum_{k\in\hat S}\hat\sigma_k\big)/\sum_{k\in S}\sigma_k$, which retains magnitude but collapses the physically incomparable families into one dimension.

\paragraph{Reward-layer fidelity controls.}
The reward-layer comparison in the main text (geometric confusion $0.000$ vs.\ $0.25$ for count over $24/24$ scenarios) is stable under four controls (Fig.~\ref{fig:reward-controls}). \emph{(i) On-policy distribution.} Scoring confusion on the states visited under three rollout policies---greedy on the geometric reward, greedy on the count reward, and uniform random---keeps the geometric reward at $0.000$ while count ranges $0.21$--$0.22$, so the separation is a property of the reward rather than of a chosen visitation. \emph{(ii) Complexity scaling.} On controlled trap states with $B$ simultaneous violation branches, count confusion grows from $0.347$ at $B{=}2$ to $0.465$ at $B{=}10$ (Pearson $r{=}0.889$) as more distinct actions collapse onto the same count-delta level, while the geometric reward stays at $0.000$: the representational gap widens with problem complexity. \emph{(iii) Exactness.} The simulator-computed geometric reward retains its fidelity advantage under additive Gaussian perturbation of scale $s\sigma$: its misordering rate rises from $0.011$ ($s{=}0$) past the count projection's $0.157$ only once $s$ exceeds ${\approx}0.8$---the margin exact simulation affords over a learned approximation. \emph{(iv) Exploit resistance.} Fidelity also bounds how easily the reward is gamed: a binary admit/reject reward pays $+0.5$ for both a \emph{stall}, an admitted no-op with $\hat s{=}s$, and a \emph{drift}, a count-preserving reallocation that worsens a branch. An $8$-step stall therefore accrues $4.0$ against $1.5$ for recovering in three steps. The count reward is likewise indifferent to drift, whereas the severity-scalar and geometric rewards both score ${\approx}0$ on a stall and penalize drift. Resistance to both hacks rises with the constraint geometry the reward preserves.

\begin{figure*}[t]
\centering
\includegraphics[width=0.86\textwidth]{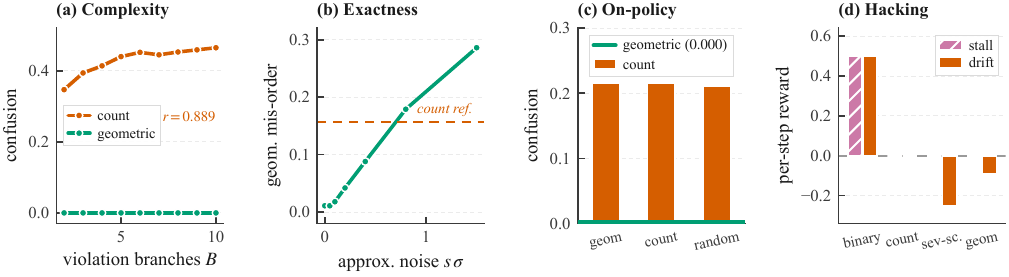}
\caption{Reward-layer fidelity controls. (a)~Reward-level confusion grows with the number of simultaneous violation branches $B$ for the count reward (Pearson $r{=}0.889$) while the geometric reward stays at $0.000$. (b)~The geometric reward retains its advantage over the count projection under additive noise up to $s{\approx}0.8\,\sigma$. (c)~The fidelity gap holds on states visited under three rollout policies (geometric confusion $0.000$ throughout; count $0.21$--$0.22$). (d)~A binary admit/reject reward pays $+0.5$ for both stall and drift; the count reward is blind to drift, which the severity-scalar and geometric rewards penalize.}
\label{fig:reward-controls}
\end{figure*}

\paragraph{The multi-family failure class.}
When several physically incomparable constraint families are active at once, any scalar reward must collapse the product order into one dimension, and the representational gap of Prop.~\ref{prop:gap} reappears in training dynamics as advantage collapse (Prop.~\ref{prop:collapse}). We run two controlled mechanism-isolation training-dynamics analyses~\citep{krakovna2020specification}, each a minimal multi-family recovery regime isolating one component of the violation geometry, with GRPO on the unmodified reward functions ($16$ training seeds each; trained recovery in Table~\ref{tab:failureclass}). In the \emph{$\sigma$-heterogeneity regime} (one large-$\sigma$ family, one small-$\sigma$ family whose violation is urgent), the severity scalar is hijacked by raw magnitude, chasing the large family and recovering $0\%$, while per-family normalization makes the urgent family's elimination worth a full family, and the geometric arm recovers $100\%$. In the \emph{coupling regime} (clearing the large family drifts the small family toward an irreversible limit), the count reward is blind to the drift its own actions cause and trains to $12\%$; the drift-aware geometric reward trains to $100\%$. Restoring the drift term rescues count in the coupling regime ($1.00$) but not cross-family priority ($0.80$ under $\sigma$-heterogeneity), and the severity scalar fails in both---the two regimes require distinct components that only the full geometric reward combines.

The $\sigma$-heterogeneity sweep (Fig.~\ref{fig:trainenh}(b)) shows the failure is not an artifact of a chosen operating point: it emerges at $\sigma_L/\sigma_S\!\approx\!4$, saturates from $8$ onward (CityLearn's measured heterogeneity is $286$), and vanishes in the homogeneous limit. The separation holds across the full $4{\times}4$ coupling--floor grid (Fig.~\ref{fig:failgrid}): geometric $\ge0.97$ in all $16$ cells, count $\le0.42$ in $15$; the single cell where the count reward retains recovery lies at the weakest coupling. An MLP policy reproduces the same ordering ($1.00/0.13/0.05$ for geometric/count/severity-scalar), so the pattern is not specific to the policy class.

\begin{figure}[t]
\centering
\includegraphics[width=\columnwidth]{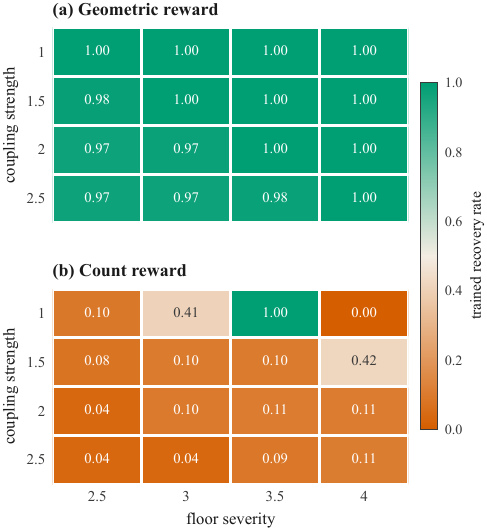}
\caption{Trained recovery across the $4{\times}4$ coupling--floor grid. The geometric reward recovers $\ge0.97$ in all $16$ cells; the count reward collapses ($\le0.42$) in $15$ of $16$, recovering only in the single cell $(\text{coupling}, c_{\text{floor}}){=}(1.0,3.5)$, at the weakest coupling in the grid.}
\label{fig:failgrid}
\end{figure}

\paragraph{The mechanism at full scale.}
Training Qwen3-8B with GRPO on the multi-family CityLearn suite (four buildings, $\sigma$-het $286$) reproduces the mechanism in a second, larger system and shows that trajectory-return credit is needed to propagate delayed cross-family consequences, since a per-step z-score cannot assign the cost of flooring the battery family back to the early step that caused it. With trajectory-return credit the trained policies separate along the axis Prop.~\ref{prop:collapse} predicts. Across seeds the geometric reward's advantage signal correlates with the true one-step action value (Spearman's $\rho{=}0.96$ vs.\ $0.79$ for count). After training, at the same violation count, the count arm leaves the battery family with $14\%$ higher residual severity than the geometric arm, which is the signature of a severity-blind signal.

\bibliography{refs}

\end{document}